\documentclass{article}
\usepackage{nips07submit_e,times}
\usepackage{algorithmic}
\usepackage{algorithm}
\usepackage{mathtools}
\usepackage{graphicx} 
\usepackage{subfigure}
\usepackage{amsthm}

\newcommand{\W}[3]{
{\stackrel{{#1}}{W}} \,_{#2}^{#3}
}
\def\1{{\mathbf 1}}

\newtheorem{thm}{Theorem}

\title{Online Coordinate Boosting}

\author{Raphael Pelossof \\
Department of Computer Science\\
Columbia University\\
2960 Broadway, New York, NY 10027 \\
\texttt{pelossof@cs.columbia.edu} \\
\And
Michael Jones \\
Mitsubishi Electric Research Labs \\
201 Broadway,
Cambridge, MA 02139 \\
\texttt{mjones@merl.com} \\
\AND
Ilia Vovsha \\
Columbia University\\
2960 Broadway, New York, NY 10027\\
\texttt{iv2121@columbia.edu} \\
\And
Cynthia Rudin \\
Columbia University\\
Center for Computational Learning Systems\\
Interchurch Center, 475 Riverside Drive MC 7717\\
New York, NY 10115\\
\texttt{rudin@ccls.columbia.edu} \\
}

\begin{document}

\maketitle

\begin{abstract}
We present a new online boosting algorithm for adapting the weights of a boosted classifier, which yields a closer approximation to Freund and Schapire's AdaBoost algorithm than previous online boosting algorithms.
We also contribute a new way of deriving the online algorithm that ties together previous online boosting work.
We assume that the weak hypotheses were selected beforehand, and only their weights are updated during online boosting.  
The update rule is derived by minimizing AdaBoost's loss when viewed in an incremental form.
The equations show that optimization is computationally expensive. However, a fast online approximation is possible.
We compare approximation error to batch AdaBoost on synthetic datasets and generalization
error on face datasets and the MNIST dataset.
\end{abstract} 
\section{Introduction}
\label{submission}
Most practical algorithms for object detection or classification
require training a classifier that is general enough to work in almost
any environment.  Such generality is often not needed once
the classifier is used in a real application.  A face detector, for example, 
may be run on a fixed camera data stream and
therefore not see much variety in non-face patches.  Thus, it would be desirable 
to adapt a classifier in an online fashion to achieve
greater accuracy for specific environments. In addition, the target concept
might shift as time progresses and we would like the classifier
to adapt to the change. Finally, the stream 
may be extremely large which deems batch-based 
algorithms to be ineffective for training.

Our goal is to create a fast and accurate online learning algorithm that can adapt an
existing boosted classifier to a new environment and concept change.
This paper looks at the core problem that
must be solved to meet this goal which is to develop
a fast and accurate sequential online learning algorithm. 
We use a traditional online learning approach, which is to assume that the feature mapping is selected beforehand and is fixed while training. The paradigm allows us to adapt our algorithm easily to a new environment. 
The algorithm is derived by looking at the minimization
of AdaBoost's exponential loss function when training AdaBoost 
with N training examples, then adding a single example to the training
set, and retraining with the new set of $N+1$ examples.  
The equations show that an online algorithm that exactly
replicates batch AdaBoost is not possible, since the update requires computing 
the classification results of the full dataset by all the weak hypotheses. We show that a simple and
approximation that avoids this costly computation is possible, resulting in a fast online algorithm. 
Our experiments show that by greedily minimizing the approximation error at each coordinate we are able to approximate batch AdaBoost better than Oza and Russell's algorithm.

The paper is organized as follows, in section 2 we discuss related work. In section 3 we present AdaBoost in exact incremental form, then we derive a fast approximation to this form, and discuss issues that arise when implementing the approximation as an algorithm. We also compare our algorithm with Oza and Russell's algorithm \cite{oza01online}. We conclude with experiments and a short discussion in section 4.
\section{Related Work}
The problem of adapting the weights of existing classifiers is a topic of ongoing 
research in vision \cite{huang07incremental,javed05online, pham07online, wu07improving}.
Huang et al's \cite{huang07incremental} work is most closely related to our work.
They proposed an incremental learning algorithm to update the weight of each weak hypothesis.
Their final classifier is a convex combination of an offline model
and an online model. Their offline model is trained solely on offline examples, and is based on a similar 
approximation to ours. Our model combines both their models into one uniform model, which does not differentiate between offline and online examples. This allows us to continuously adapt
regardless of whether or not the examples were seen in the offline or online
part of the training. Also, by looking at the change in
example weights as a single example is added to the training set, we are able to 
compute an exact update to the weak hypotheses weights, in an online manner, that does not require a line search as in Huang et al's work.

Our online algorithm stems from an approximation to AdaBoost's loss minimization
as the training set grows one example at a time. 
We use a multiplicative update rule to adapt the classifier weights.
The multiplicative update for online algorithms was first proposed by Littlestone 
\cite{littlestone88learning} with the Winnow algorithm. Kivinen and Warmuth 
\cite{kivinen97exponentiated} extended the update rule
of Littlestone to achieve a wider set of classifiers by incorporating 
positive and negative weights. Freund and Schapire \cite{freund97decision}
converted the online learning paradigm to batch learning with multiplicative weight updates.
Their AdaBoost algorithm keeps two sets of weights, one on the data and one on the weak hypotheses.
AdaBoost updates the example weights at each training round to form a harder problem for the next round.
This type of sequential reweighting in an online setting, where only one example is kept at any time,
was later proposed by Oza and Russell \cite{oza01online}.
They update the weight of each weak hypothesis 
sequentially. 
At each iteration, a weak hypothesis classifies a weighted example, where the example's weight is derived from the performance of the current combination of weak hypotheses.
Like our algorithm, Oza and Russell's algorithm has a sequential update for the weights of the weak hypotheses, 
however, unlike ours, theirs includes feature selection. Our algorithm is also derived from the more recent
AdaBoost formulation \cite{schapire99improved}. 
We show how the Online Coordinate Boosting algorithm weight update rule can be reduced to Oza and
Russell's update rule with a few simple modifications.

Both our and Oza and Russell's algorithms store for each classifier an approximation of the sums of example weights that were correctly and incorrectly classified by each weak hypothesis. They can be seen as algorithms for estimating the weighted error rate of each weak hypothesis under memory and speed constraints. 
Another algorithm that can be seen this way is Bradley and Schapire's FilterBoost algorithm \cite{bradley07filterboost}. FilterBoost uses nonmonotonic
adaptive sampling together with a filter to sequentially estimate
the edge, an affine transformation of the weighted error, of each weak hypothesis. When the edge
is estimated with high probability the algorithm updates its classifier and continues
to select and train the next weak hypothesis. 
Unlike our and Oza and Russell's algorithm, FilterBoost cannot adapt already selected weak hypotheses weights to drifting concepts.

\section{Online Coordinate Boosting}
We would like to minimize batch AdaBoost's bound on the error using a fast update
rule as examples are presented to our algorithm. Let $(x_1,y_1),..,(x_{N+1},y_{N+1})$ be a stream of labeled examples $x_i\in {\cal R}^M, y_i\in\{-1,1\}$, and let a classifier be defined by a linear combination of weak hypotheses $H(x) = sign(\sum_{j=1}^J \alpha_j h_j(x))$, where the weights are real-valued $\alpha_j\in{\cal R}$ and each weak hypothesis $h_j$ is preselected and is binary $h_j(x)\in\{-1,1\}$. 
We use the term \textit{coordinate} as the index of a weak hypothesis. Let $m_{ij}=y_ih_j(x_i)$ be defined as the margin which is equal to 1 for correctly classified examples and -1 for incorrectly classified examples by weak hypothesis $j$. 
Throughout training, AdaBoost maintains a weighted distribution over the examples. The weights at each time step are set to minimize the classification error according to batch AdaBoost \cite{schapire98boosting}. Adding a single example to the training set changes the weights of the examples, and the weights of the entire classifier. AdaBoost defines the weight of example $i$ as $d_{iJ}=e^{-\sum_{j=1}^{J-1} \alpha_j m_{ij}}$, which implies $d_{iJ} = d_{i,J-1} e^{-\alpha_{J-1}m_{i,J-1}}$. Furthermore, the weight of a weak hypothesis $J$ is defined as $\alpha_J=\frac{1}{2}\log{W_J^+}/{W_J^-}$, where the sums of correctly and incorrectly classified examples by weak hypothesis $j$ are defined by $W_{J}^+=\sum_{i:m_{iJ}=+1}d_{iJ}$ and $W_J^-=\sum_{i:m_{iJ}=-1}d_{iJ}$ correspondingly. We define $\1_{[\:]}$ as the indicator function.

We use superscript to indicate time, which in the batch setting is the number of examples in the training set, and in the online setting is the index of the last example. To improve legibility, if we drop the superscript from an equation, the time index is assumed to be $N+1$.
Therefore, when adding the $N+1$ example, the weights of the other examples will change from $d_{iJ}^N$ to $d_{iJ}^{N+1}$ and the weights of each weak hypothesis from $\alpha_J^N$ to $\alpha_J^{N+1}$. We denote the change in a weak hypothesis weights as $\Delta\alpha_J^N = \alpha_J^{N+1}-\alpha_J^N$.

\subsection{AdaBoost in exact incremental form}
AdaBoost's loss function $Z_{J+1}=\sum_i d_{iJ}e^{-\alpha_{J} m_{iJ}}$ bounds the training error. It has been shown \cite{schapire98boosting, schapire99improved} that minimizing this loss tends to lower generalization error. We are motivated to minimize a fast and accurate approximation to the same loss function, as each example is presented to our algorithm. 
Similarly to AdaBoost, we fix all the coordinates up to coordinate $J$, and seek to minimize the approximate loss at the $J$th coordinate. The optimization is done by finding the update  $\Delta\alpha_{J}^N$ that minimizes AdaBoost's approximate loss with the addition of the last example. More formally, we are given the previous weak hypotheses weights $\alpha_1^N,..,\alpha_{J}^N$ and their updates so far $\Delta\alpha_1^N,..,\Delta\alpha_{J-1}^N$ and wish to compute the update $\Delta\alpha_{J}^N$ that minimizes $Z_{J+1}$. The resulting update rule is the change we would get in coordinate $J$'s weight if we trained batch AdaBoost with N examples and then added a new example and retrained with the larger set of $N+1$ examples. Looking at the derivative of batch AdaBoost's loss function when adding a new example, we get the update rule for $\Delta\alpha_{J}^N$:
\begin{eqnarray}
Z_{J+1} &=& \sum_{i=1}^{N+1} d_{iJ} e^{-\alpha_{J}^{N+1}m_{iJ}}
=\sum_{i=1}^{N+1} d_{iJ} e^{-(\alpha_{J}^{N}+\Delta\alpha_{J}^N)m_{iJ}}\\
\frac{\partial Z_{J+1}}{\partial\Delta\alpha_{J}^N}&=&-\sum_{i=1}^{N+1} d_{iJ} e^{-(\alpha_{J}^{N}+\Delta\alpha_{J}^N)m_{iJ}}m_{iJ}\\
&=& \sum_{i:m_{iJ}=-1}d_{iJ} e^{(\alpha_{J}^N+\Delta\alpha_{J}^N)}-\sum_{i:m_{iJ}=+1}d_{i,J} e^{-(\alpha_{J}^N+\Delta\alpha_{J}^N)}\\
&=& \W{-}{J}{N+1}e^{(\alpha_{J}^N+\Delta\alpha_{J}^N)} - \W{+}{J}{N+1}e^{-(\alpha_{J}^N+\Delta\alpha_{J}^N)}.
\end{eqnarray}
Setting the derivative to zero and solving for $\Delta\alpha_{J}^N$ we get:
\begin{eqnarray}
\label{eqn:alpha_update}
\Delta\alpha_J^N=\frac{1}{2}\log\frac{\W{+}{J}{N+1}}{\W{-}{J}{N+1}}-\alpha_J^N
\doteq \alpha_J^{N+1}-\alpha_J^N.
\end{eqnarray}
The update $\Delta\alpha_J^N$ that minimizes $Z_{J+1}$ is dependent on two quantities $\W{+}{J}{N+1}$ and $\W{-}{J}{N+1}$. These are the sums of weights of examples that were respectively classified correctly and incorrectly by weak hypothesis $J+1$, when training with $N+1$ examples. 

We rewrite these sums in an incremental form. The incremental form is derived by separating the weight of the last example that was added to each of the sums from the rest of the sum. This will allow us later on to compute a fast incremental approximation to them, resulting in our online algorithm. 
We combine the analysis of both sums by incorporating the parameter $\sigma\in\{-1,+1\}$, which represents the sign of the margin of the examples being grouped by the cumulative sum.
Formally, we will break these subsets to subsets over N weights $\{d_{1J},..,d_{NJ}\}$, and the weight of the last example $d_{N+1,J}$ which is added to the appropriate sum using the function $g_J^\sigma = d_{N+1,J}\1_{[m_{N+1,J}=\sigma]}$:
\begin{eqnarray}
\W{\sigma}{J}{N+1} &=& \sum_{i:m_{iJ}=\sigma}d_{iJ} 
=\sum_{i_{J}^\sigma}d_{iJ} + g_{J}^\sigma
=\sum_{i_{J}^\sigma}\prod_{j=1}^{J-1} e^{- \alpha_{j}^{N+1}m_{ij}} + g_{J}^\sigma\\
&=&\sum_{i_{J}^\sigma}\prod_{j=1}^{J-1} e^{- (\alpha_j^N+\Delta\alpha_j^N)m_{ij}}  + g_{J}^\sigma
\label{eqn:pseudoincremental_pos}
=\sum_{i_{J}^\sigma} d_{iJ}^{N} \prod_{j=1}^{J-1} e^{-\Delta\alpha_j^N m_{ij}} + g_J^\sigma.
\end{eqnarray}
We define the subsets of examples as $i_{J}^\sigma =\{ i |(m_{iJ}=\sigma) \wedge (i\le N)\}$.
We partition the indices of the first $N$ examples to two subsets: a subset of correctly classified examples, where $\sigma=+1$, and incorrectly classified examples, where $\sigma=-1$.
\subsection{A fast approximation to the incremental form}
Equation \ref{eqn:pseudoincremental_pos} is a sum product expression which is
costly to compute and requires that the margins of all previous
examples be stored. In order to make this an online algorithm which stores only one example in the memory,
we approximate each term in the product with a term that is independent of
all of the margins $m_{ij}$. This type of approximation enables us to separate
the sum of the weights from the product terms, which results in a faster approximate update rule:
\begin{eqnarray}
\label{eqn:sumprod}
\W{\sigma}{J}{N+1} &=& \sum_{i_{J}^\sigma} d_{iJ}^{N} \prod_{j=1}^{J-1} {e^{-m_{ij} \Delta\alpha_j^N}}+ g_{J}^\sigma\\
\label{eqn:sumprodapprox}
&\approx& \W{\sigma}{J}{N}\prod_{j=1}^{J-1} (q_{jJ}^\sigma e^{-\Delta\alpha_j^N} + (1-q_{jJ}^\sigma)e^{\Delta\alpha_j^N})  +g_{J}^\sigma
\end{eqnarray}
where $q_{jJ}^\sigma\in {\cal R}$. The transition from equation \ref{eqn:sumprod} to \ref{eqn:sumprodapprox} is done in two steps. The terms in the product are approximated by new terms that are independent of $i$. Given this independence, the sum of weighted examples can be grouped to the cumulative sum of previous weights. Equation \ref{eqn:sumprodapprox} is very similar to Huang et al's offline loss function. However, by greedily solving the approximation error equations, we show that the update to the model should take into account all the examples, and not just the offline ones as in \cite{huang07incremental}.

Since our approximation incurs errors, we would like to find for each weak hypothesis the parameters $q_{jJ}^\sigma$ that minimize the approximation error. Equation \ref{eqn:sumprodapprox} can be rewritten in two equivalent forms to show two types of errors:
\begin{eqnarray}
\W{\sigma}{J}{N+1}&\approx&\sum_{i_{J}^\sigma} d_{iJ}^N \prod_{j=1}^{J-1}( e^{\Delta\alpha_j} + q_{jJ}^\sigma (e^{-\Delta\alpha_j} - e^{\Delta\alpha_j}))+ g_J^\sigma \label{eqn:approx_err1}\\
\label{eqn:approx_err2}
&=&\sum_{i_{J}^\sigma} d_{iJ}^N \prod_{j=1}^{J-1}( e^{-\Delta\alpha_j} + (1-q_{jJ}^\sigma) (e^{\Delta\alpha_j} - e^{-\Delta\alpha_j})) + g_J^\sigma.
\end{eqnarray}
The equivalent approximation forms give us a way to compute the exact error for any choice of $m_{ij}$. However, the exact error expression may have $2^J$ terms and exactly minimizing it may be costly. Instead, by taking a greedy approach and looking at a part of the error terms we are able to minimize the approximation error at each coordinate. We formulate the problem as follows: nature chooses a set of margins $m_{ij}$ and the booster chooses $q_{1J}^\sigma,..,q_{jJ}^\sigma$ to minimize the approximation error of the boosted classifier at each coordinate. Let $\delta_j=e^{-\Delta\alpha_j}-e^{\Delta\alpha_j}$, then for each weak hypothesis, if nature choses the margin $m_{ij}=-1$, according to \ref{eqn:approx_err1}, the squared error at coordinate $j$ is $(q_{jJ}^\sigma)^2\delta_j^2$. If nature chooses a margin $m_{ij}=+1$, then according to \ref{eqn:approx_err2}, the squared error at coordinate $j$ is $ (1-q_{jJ}^\sigma)^2\delta^2$. We look at squared error to avoid negative errors. Regardless of the choice of margin, we can only make one type of error since the margins are binary. 

Theorem \ref{thm:approxsol} gives us the solution for parameters $q_{jJ}^\sigma$ using  a greedy minimization of the weighted squared approximation error at each coordinate.
\begin{thm}
\label{thm:approxsol}
Let the weighted squared approximation error at coordinate $j$ and sign $\sigma$ be defined by
\begin{equation}
\epsilon_{jJ}^\sigma =\sum_{i_{J}^\sigma} d_{iJ}^N  \left({\bf 1}_{[i:m_{ij}=-1]} (q_{jJ}^\sigma)^2\delta_j^2 + {\bf 1}_{[i:m_{ij}=+1]} (1-q_{jJ}^\sigma)^2\delta_j^2\right).
\end{equation}
Then, the minimizer $q_{jJ}^\sigma$ of the weighted approximation error at coordinate $j$ is:
\begin{equation}
q_{jJ}^\sigma =\frac{ \sum_{i_{J}^\sigma \wedge i_j^+}d_{iJ}^N}{\sum_{i_{J}^\sigma} d_{iJ}^N}.
\end{equation}
\end{thm}
\begin{proof}
Using a greedy approach and looking at the weighted squared approximation error at a single coordinate $j$ given the weights of the examples at coordinate $J$, we solve for $q_{jJ}^\sigma$. Since the error function is convex, we can take derivatives and solve to find the global minimum:
\begin{eqnarray}
\frac{\partial \epsilon_{jJ}^\sigma}{\partial q_{jJ}^\sigma}&=&2\delta_j^2 \sum_{i_{J}^\sigma} d_{iJ}^N \left({\bf 1}_{[i:m_{ij}=-1]} q_{jJ}^\sigma - {\bf 1}_{[i:m_{ij}=+1]} (1-q_{jJ}^\sigma)\right).
\end{eqnarray}
We solve for $q_{jJ}^\sigma$ by setting the derivative to zero. We can divide by $\delta_j$ since all the example weights are positive and therefore $\delta_j\ne0$.
\begin{eqnarray}
q_{jJ}^\sigma = \frac{ \sum_{i:i_{J}^\sigma \wedge m_{ij}=+1}d_{iJ}^N}{\sum_{i:i_{J}^\sigma \wedge m_{ij}=-1} d_{iJ}^N +  \sum_{i:i_{J}^\sigma \wedge m_{ij}=+1} d_{iJ}^N}= \frac{ \sum_{i_{J}^\sigma \wedge i_j^+}d_{iJ}^N}{\sum_{i_{J}^\sigma} d_{iJ}^N}.
\end{eqnarray}
\end{proof}
Theorem \ref{thm:approxsol} has a very natural interpretation. The minimizer $q_{jJ}^\sigma$ can be seen as the weighted probability of weak hypothesis $j$ producing a positive margin and weak hypothesis $J$ producing a margin $\sigma$ (either positive or negative.)
\subsection{Implementing the approximation as an algorithm}
{\bf Initialization:} 
The recursive form of equation \ref{eqn:sumprodapprox} requires us to define a setting for $\W{\sigma}{0}{N}$. Let $\W{\sigma}{0}{N}=|i_1^\sigma |$ be the count of examples with a $\sigma$ margin with the first weak hypothesis.
This is equivalent to setting the initial weight of each example to one, which gives all the examples equal weight before being classified by the first weak hypothesis.\\
{\bf Weight updates:} 
Theorem \ref{thm:approxsol} shows that calculating the error minimizer requires keeping sums of weights which involve two weak hypotheses $j$ and $J$. Similarly to our approximation of the sums of weights  $\W{\sigma}{J}{N+1}$, we need to approximate $q_{jJ}^\sigma$ as examples are presented the the online algorithm. Applying the same approximation to estimate $q_{jJ}^\sigma$ yields the a similar optimization problem, however the approximation error minimizers for this problem involves three margins. We avoid calculating this new minimizer, and instead use the same correction we used for  $\W{\sigma}{J}{N+1}$ (see Algorithm \ref{alg:ocb}.)\\
{\bf Running time:} 
Retraining AdaBoost for each new example would require $O(N^2J)$ operations, as the classifier needs to be fully trained for each example. By using our approximation we can train the classifier in $O(NJ^2)$, where the processing of each example takes $O(J^2)$. A tradeoff between accuracy and speed can be established by only computing the last $K$ terms of the product, and assuming that the others are equal to one. This speedup results in running time complexity  $O(NJK)$, where $K$ will be defined as the order of the algorithm. Algorithm \ref{alg:ocb} shows the Online Coordinate Boosting algorithm with order $K$.
\begin{algorithm}[tb]
   \caption{K-order Online Coordinate Boosting}
   \label{alg:ocb}
\begin{algorithmic}
   \STATE {\bfseries Input:} Example classifications $M\in\{-1,1\}^{N\times J}$ where $m_{ij} = y_i h_j(x_i)$
   \STATE	\hspace{0.39in} Order paramenter $K$
   \STATE	\hspace{0.39in} Smoothing parameter $\epsilon$
   \STATE Option 1: Initialize $\alpha_j=0,\Delta\alpha_j=0$ where $j=0,..,J$.
   \STATE \hspace{1in} $W_{jk}^+=\epsilon, W_{jk}^-=\epsilon$ where $j,k=0,..,J$
   \STATE Option 2: Initialize using AdaBoost on a small set.
   \FOR{$i=1$ {\bfseries to} $N$}
   	\STATE $d=1$
	\FOR{$j=1$ {\bfseries to} $J$}
		\STATE $j_0 = \max(0,j-K)$
		\STATE $\pi_j^+ = \prod_{k=j_0}^{j-1} \left( \frac{W_{jk}^{+}}{W_{jj}^+} e^{-\Delta\alpha_k} + (1- \frac{W_{jk}^{+}}{W_{jj}^+}) e^{\Delta\alpha_k}\right)$
		\STATE $\pi_j^- =  \prod_{k=j_0}^{j-1} \left(  \frac{W_{jk}^{-}}{W_{jj}^-} e^{-\Delta\alpha_k} + (1- \frac{W_{jk}^{-}}{W_{jj}^-}) e^{\Delta\alpha_k}\right)$
		\FOR{$k=1$ {\bfseries to} $j$}
			\STATE $W_{jk}^{+} \leftarrow  W_{jk}^{+} \pi_j^+ + d\1_{[m_{ik}=+1]} \cdot \1_{[m_{ij}=+1]}$
			\STATE $W_{jk}^{-} \leftarrow  W_{jk}^{-} \pi_j^- + d\1_{[m_{ik}=-1]} \cdot \1_{[m_{ij}=-1]}$
		\ENDFOR
		\STATE $\alpha_j^{i} = \frac{1}{2}\log\frac{W_{jj}^+}{W_{jj}^-}$
		\STATE $\Delta\alpha_j = \alpha_j^{i}-\alpha_j^{i-1}$
		\STATE $d \leftarrow d e^{-\alpha_j^{i} m_{ij}}$
	\ENDFOR
   \ENDFOR
   \STATE {\bfseries Output: $\alpha_1^N,..,\alpha_J^N$}
\end{algorithmic}
\end{algorithm}
\subsection{Similarity to Oza and Russell's Online algorithm}
Let us compare Oza and Russell's algorithm \cite{oza01online} to our algorithm. Excluding feature selection, there are two steps in their algorithm. The first adds the example weight to the appropriate cumulative sum, and the second reweights the example. Step one is identical to the addition that our algorithm performs if we assume that all the terms in the product in equation \ref{eqn:sumprodapprox} are equal to one, or equivalently that $\Delta\alpha_j=0$. At step two, reweighting the example,
Oza and Russell break the update rule to two cases, one for each type of margin:
\begin{eqnarray}
m_{ij}=+1&:& \qquad
d\leftarrow d\frac{W_j^+ + W_j^-}{2W_j^+} =\frac{d+d\left(\frac{W_j^-}{W_j^+}\right)}{2} \\
m_{ij}=-1&:& \qquad
d\leftarrow d{\frac{W_j^+ + W_j^-}{2W_j^-}} = \frac{d+d\left(\frac{W_j^+}{W_j^-}\right)}{2}.
\end{eqnarray}
The two cases can be consolidated to one case when we introduce the margin into the equations.
Interestingly, this update rule smooths the examples weights by taking the average
between the old weight and the new updated weight that we would get by AdaBoost's exponential reweighting \cite{schapire99improved}:
\begin{eqnarray}
\label{eqn:ozasimilarity_start}
d \leftarrow \frac{d+d\left(\frac{W_j^+}{W_j^-}\right)^{-m_{ij}}}{2}= \frac{d + d e^{-2\alpha_j m_{ij}}}{2}.
\label{eqn:ozasimilarity_end}
\end{eqnarray}
If we do not perform corrections to the W's, and only add the weight of the last example to them, we reduce our algorithm to a form similar to Oza and Russell's algorithm. Since Oza and Russell use an older AdaBoost update rule, when put in an online framework, the weights in their algorithm are squared and averaged compared to our weights.
\section{Experiments and Discussion}
We tested our algorithm against modified versions of Oza and Russell's online algorithm.
The only modification was the removal of the weak hypothesis selection process. Instead we fixed a predefined set of ordered weak hypotheses. Three experiments were conducted, the first with random data, the second with the MNIST dataset, and the third with a face dataset. Throughout all our experiments we initialized our algorithm with the cumulative weights that were produced by running AdaBoost on a small part of the training set.
We needed to initialize our algorithm to avoid divide-by-zero errors when only margins of one type have been seen for small numbers of training examples. We similarly initialized Oza and Russell's algorithm, however, since our training sets are large, it had little influence on their algorithm's performance compared to a non-initialized run.
 \begin{figure}[th]
  \begin{center}
    \subfigure[Synthetic: Average approximation error as the number of training examples is increased. Concept drift every $10K$ examples. Averaged over 5 runs. Accuracy improves with higher order.]{\label{fig:synthetic}\includegraphics[width=2.5in,height=1.51in]{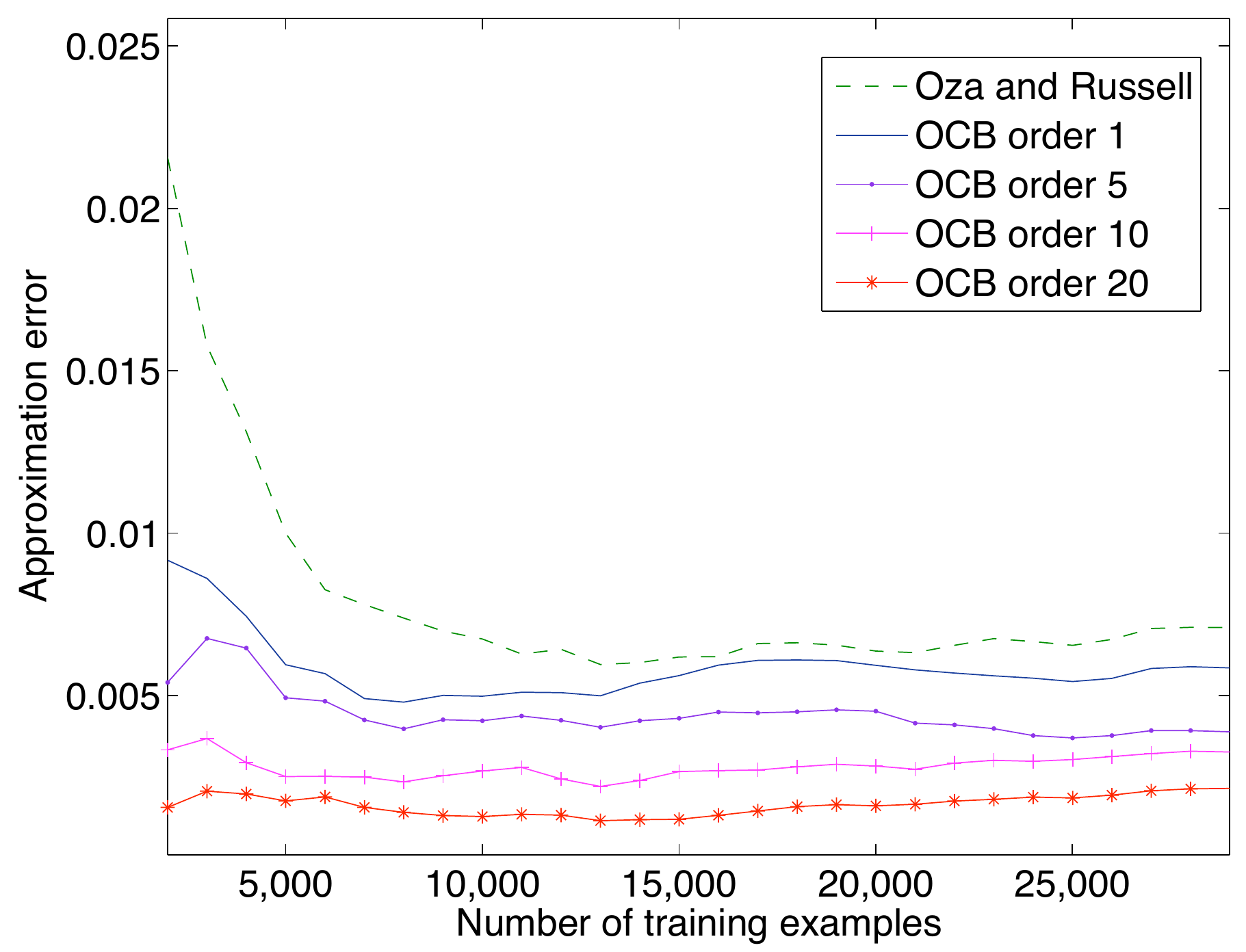}} \qquad
    \subfigure[MNIST: Combined classifier test error as the number of training examples is increased. OCB and AdaBoost achieve lower test error rates than Oza and Russell's algorithm.]{\label{fig:mnist_comb}\includegraphics[width=2.5in,height=1.51in]{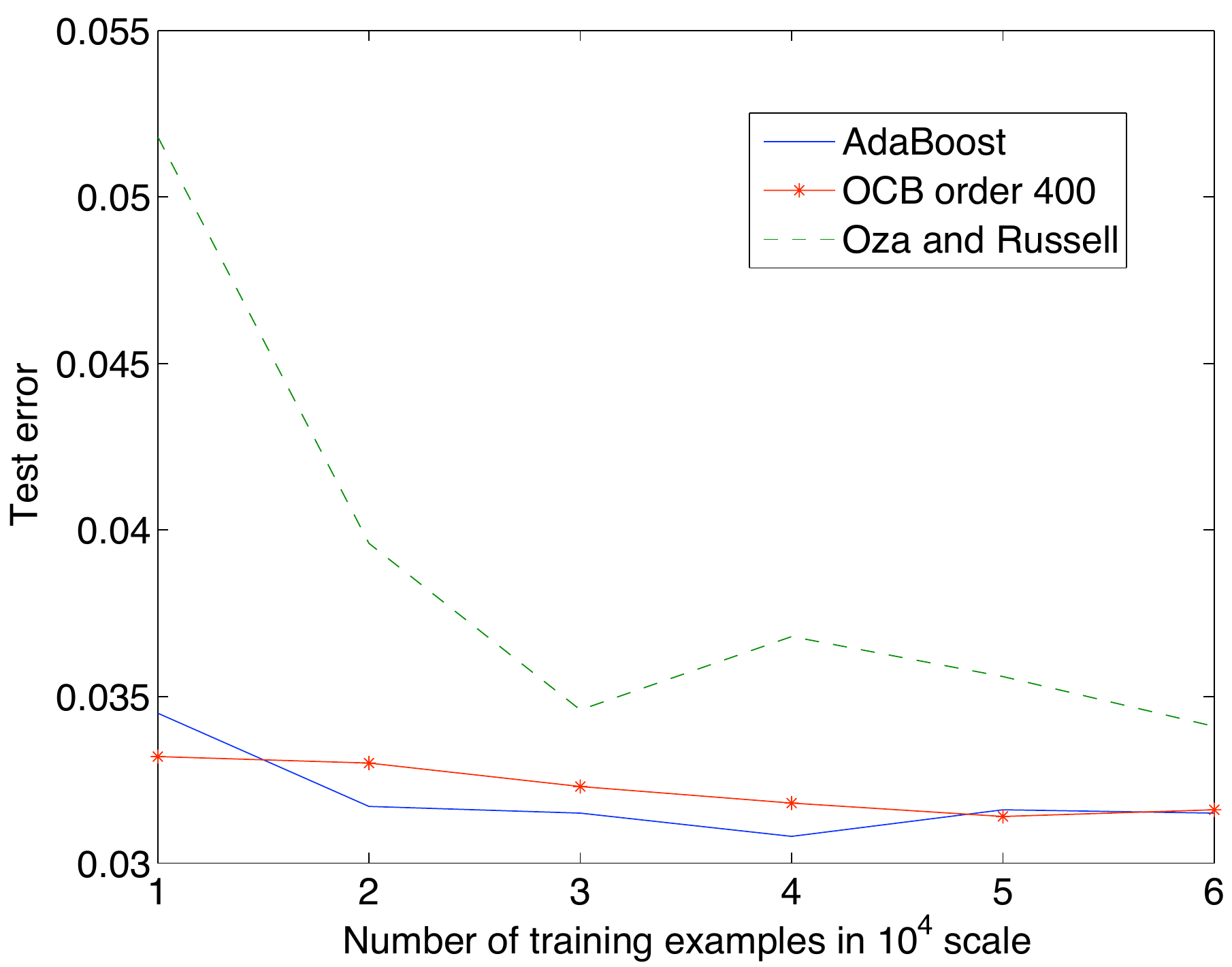}} \\
    \subfigure[Face data: Average normalized approximation error as the number of training examples is increased. Averaged over 10 permutations of the training set. OCB best approximates AdaBoost.]{\label{fig:face_approx_error}\includegraphics[width=2.5in,height=1.51in]{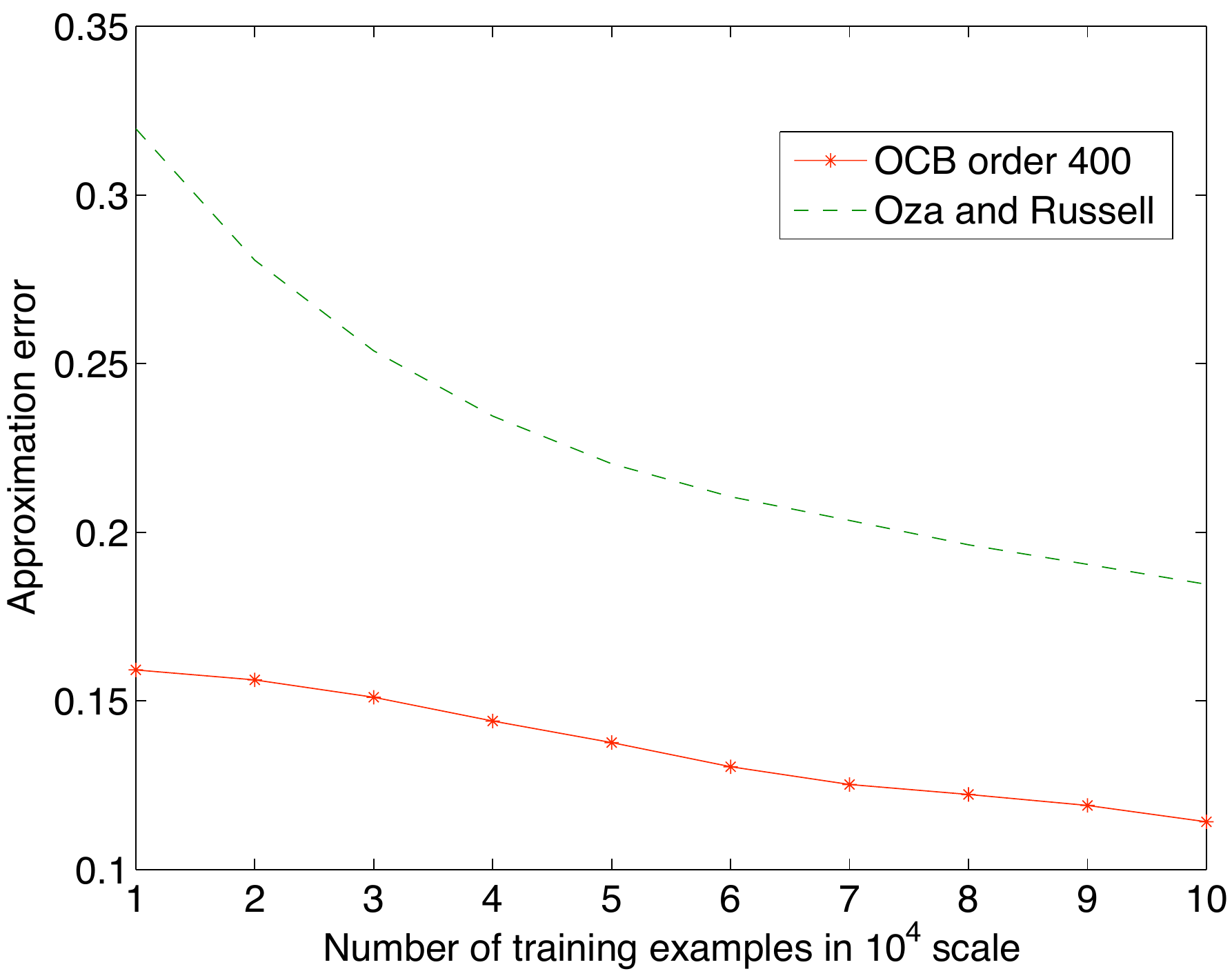}} \qquad
    \subfigure[Face data: Average 1-AUC as the number of training examples is increased. OCB and AdaBoost have almost identical performance on 100K test set.]{\label{fig:face_test_error}
\includegraphics[width=2.5in,height=1.51in]{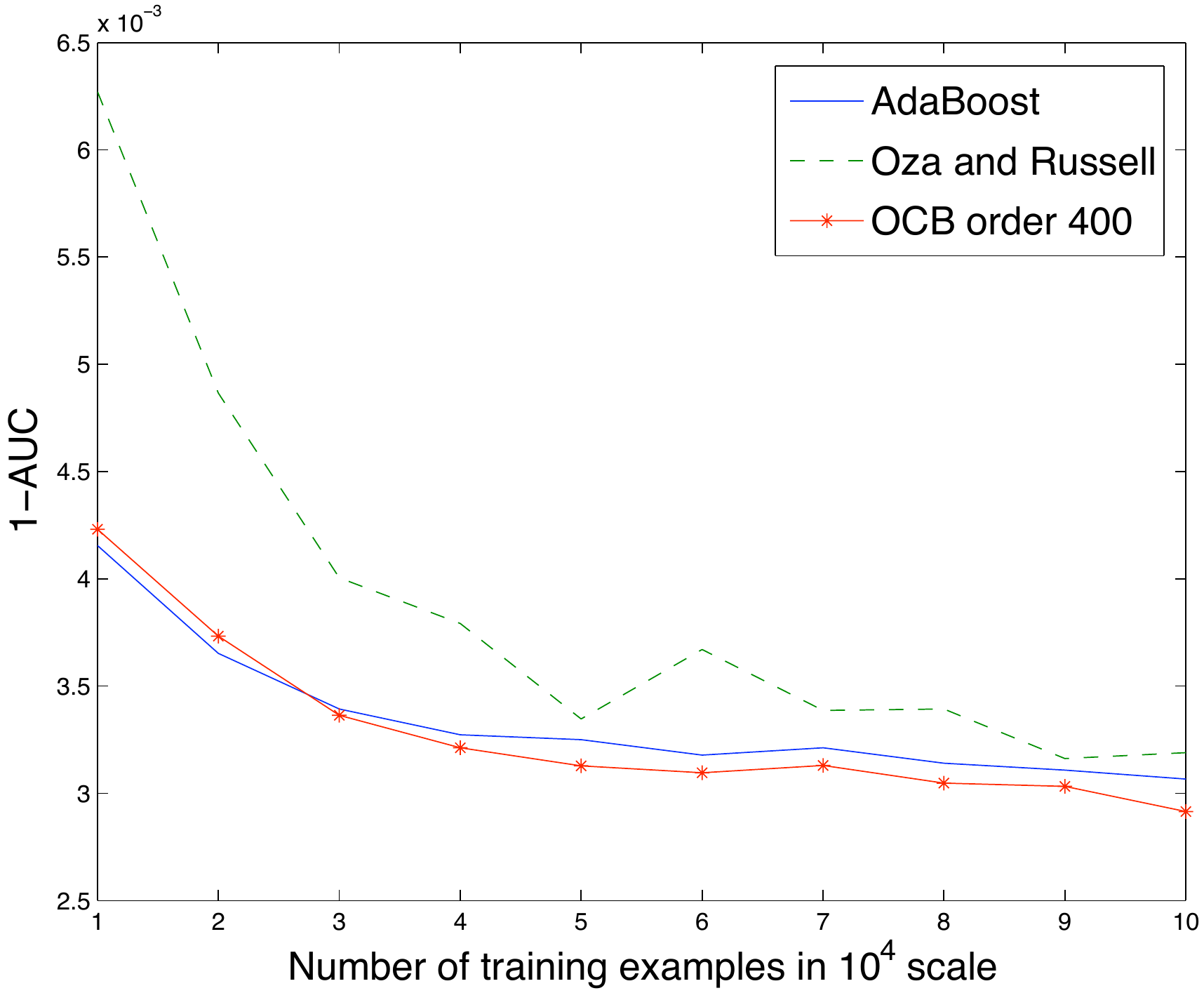}}
  \end{center}
  \caption{Approximation and Test error experiments}
  \label{fig:experiments}
\end{figure}
\\
\textbf{Synthetic data:} The synthetic experiment was set up to test the adaptation of our algorithm to concept change, and the effects of the algorithm's order on its approximation error. We created synthetic data by randomly generating multiple margin matrices $M_t$ which contain margins $m_{ij}$. Each matrix was created one column at a time where we draw a random number between zero and one for each column. The random number gives us the probability of the weak hypothesis classifying an example correctly. To simulate concept drift, each matrix $M_t$ was generated by perturbing the probabilities of the previous matrix by a small amount and sampling new margins accordingly.
We consider the normalized approximation error of the classifier learned by the online algorithms and the equivalent boosted classifier. Let the normalized approximation error between AdaBoosts's weight vector and another weight vector be defined by $err(\alpha_{ada},\alpha)=0.5\| \frac{\alpha_{ada}}{\|\alpha_{ada}\|_1} - \frac{\alpha}{\|\alpha\|_1} \|_1$.
We compared the approximation error for each example that was presented to the online algorithms with the equivalently trained batch classifier. 
The experiment was repeated $5$ times with different margin generation probabilities. Each experiment comprised of three $M_t$ matrices of size $10,000\times 20$, thereby simulating concept drift every $10,000$ examples.
Figure \ref{fig:synthetic} shows the average approximation error as the number of training examples is increased. Increasing our algorithm's order shows improvement in performance. However, we have witnessed that a tradeoff exists when training large classifiers, where the approximation deteriorates as the order is increased too much. The tradeoff exists since $q_{jJ}^\sigma$ is a greedy error minimizer, and might not optimally minimize the total approximation error.
%
\\ \textbf{Face data:}
We conducted a frontal face classification experiment using the features from an existing face detector. These weak hypotheses are thresholded box filter decision stumps. The trained face detector contains $1520$ weak hypotheses, which were learned using batch AdaBoost with resampling \cite{jones03face,viola01rapid}. Using the existing set of weak hypotheses, we compared the different online algorithms for approximation and generalization 
error on new training and test sets. Both our training and test sets consist of $93,000$ non-face images collected from the web, and $7,000$ hand labeled frontal faces all of size $24\times 24$.
We created $10$ permuted training sets by reordering the examples in the original training set $10$ times. The experimental results were averaged over the $10$ sets. This was done to verify that our algorithm is robust to any ordering. Our algorithm was initialized with the cumulative sums of weights obtained by training AdaBoost with the first $5000$ examples in each training set. Initializing Oza's algorithm did not improve its performance. We compared the online algorithms to AdaBoost's while training for every $10,000$ examples. The training results in figure \ref{fig:face_test_error} show that our online algorithm with order $400$ achieves better average AUC rates than Oza and Russell's algorithm. We compare average AUC since there are far less positives in the test set.
Figure \ref{fig:face_approx_error} shows that our average approximation of AdaBoost's weak hypotheses weights is also better. We found that setting an order of $400$ with frontal face classifiers of size $1520$ works well.\\
\textbf{MNIST data:} The MNIST dataset consists of $28\times 28$ images of the digits $[0,9]$. The dataset is split into a training set which includes $60000$ images, and a test set which includes $10,000$ images. All the digits are represented approximately in equal amount in each set. Similarly to the face detector, we trained a classifier in an offline manner with sampling to find a set of weak hypotheses. When training we normalized the images to have zero mean and unit variance. We used $h_j(x) = sign(\| x_j - x \|_2  - \theta)$ as our weak hypothesis. The weak learner found for every boosting round the vector $x_j$ and threshold $\theta$ that create a weak hypothesis which minimizes the training error. As candidates for $x_j$ we used all the examples that were sampled from the training set at that boosting round.
We partitioned the multi-class problem into $10$ one-versus-all problems, and defined a meta-rule for deciding the digit number as the index of the classifier that produced the highest vote.
The generalization and approximation error rates for each classifier can be seen in tables \ref{tbl:mnist_err} and \ref{tbl:mnist_approx}. The performance of the combination rule using each of the methods can be seen in figure \ref{fig:mnist_comb}. Again, we found that order $400$ performs well.\\
\textbf{Concluding remarks:} We showed that by deriving an online approximation to AdaBoost we were able to create a more accurate online algorithm. Nevertheless, the relationship between proximity of weak hypothesis weights and generalization needs to be further studied. One of the drawbacks of the algorithm is that it usually needs to be initialized with AdaBoost on a small training set. We are investigating adaptive weight normalization, which may allow for a better initialization scheme. We are also trying to connect FilterBoost's filtering framework and feature selection with OCB to improve performance and speed. 
 
\begin{table}[t]
\begin{tabular*}{0.95\textwidth}{@{\extracolsep{\fill}}|c|c|c|c|c|c|c|c|c|c|c|}
\hline
&0&1&2&3&4&5&6&7&8&9\\
\hline
AdaBoost & 0.31&0.19&0.8&0.89&0.9&1.0&0.47&0.79&1.62&1.29 \\
\hline\hline
OCB &\bf 0.33&\bf 0.18&\bf 0.78&\bf 0.93&0.87&\bf 0.98&\bf 0.49&\bf 0.82&\bf 1.61&\bf 1.3 \\ 
\hline
Oza &0.35&0.27&0.79&1.05&\bf 0.85&1.02&0.55&0.97&1.82&1.36 \\
\hline
\end{tabular*}
\caption{MNIST test error in \% for each classifier one-vs-all}
\label{tbl:mnist_err}
\end{table}

\begin{table}[t]
\begin{tabular*}{0.95\textwidth}{@{\extracolsep{\fill}}|c|c|c|c|c|c|c|c|c|c|c|}
\hline
&0&1&2&3&4&5&6&7&8&9\\
\hline\hline
OCB &\bf 0.07&\bf 0.1&\bf 0.04&\bf 0.04&\bf 0.04&\bf 0.04&\bf 0.06&\bf 0.05&\bf 0.04&\bf 0.03\\ 
\hline
Oza &0.1&0.1&0.09&0.09&0.1&0.09&0.1&0.1&0.1&0.09\\ 
\hline
\end{tabular*}
\caption{MNIST approximation error for each classifier one-vs-all}
\label{tbl:mnist_approx}
\end{table}

\small
\bibliography{online-review}
\bibliographystyle{plain}

\end{document}